\documentclass{article}
\newif\ifpreprint
\preprintfalse

\pdfoutput=1

\usepackage[final]{nips_2018}
\setcitestyle{numbers,square,citesep={,},aysep={,},yysep={,}}

\usepackage[utf8]{inputenc}
\usepackage[T1]{fontenc}
\usepackage{hyperref}
\usepackage{url}
\usepackage{booktabs}
\usepackage{amsfonts}
\usepackage{nicefrac}
\usepackage{microtype}
\usepackage{xr-hyper}

\usepackage{graphicx}
\usepackage{listings}
\usepackage{mathtools}
\usepackage{amssymb}
\usepackage{amsmath}
\usepackage{amsthm}
\usepackage{docmute}
\usepackage[acronym, nomain]{glossaries}
\usepackage{microtype}
\usepackage{subfigure}
\usepackage[ruled, vlined]{algorithm2e}

\newcommand{\github}{https://github.com/ytsmiling/lmt}
\newcommand{\Sec}{Sec.}

\ifpreprint
\newcommand{\appendixing}[1]{Appendix{#1}}
\else
\newcommand{\appendixing}[1]{Appendix{#1}}
\fi
\newtheorem{lemma}{Lemma}
\newtheorem{theorem}{Theorem}
\newtheorem{prop}{Proposition}
\newtheorem{corollary}{Corollary}
\DeclareMathOperator*{\argmax}{\mathop{\rm argmax}}
\DeclarePairedDelimiterX{\KL}[2]{\text{KL}[}{]}{%
#1\;\delimsize\|\;#2%
}
\newcommand{\rmnum}[1]{\romannumeral #1}

\usepackage{xr}
\begin{document}
    
    \title{Lipschitz-Margin Training: Scalable Certification of\\
    Perturbation Invariance for Deep Neural Networks}

        \author{
            Yusuke Tsuzuku\\
            The University of Tokyo\\
            RIKEN\\
            \texttt{tsuzuku@ms.k.u-tokyo.ac.jp}\\
            \And
            Issei Sato\\
            The University of Tokyo\\
            RIKEN\\
            \texttt{sato@k.u-tokyo.ac.jp}\\
            \And
            Masashi Sugiyama\\
            RIKEN\\
            The University of Tokyo\\
            \texttt{sugi@k.u-tokyo.ac.jp}
        }

    \maketitle

        \begin{abstract}

        High sensitivity of neural networks against malicious perturbations on inputs causes security concerns.
        To take a steady step towards robust classifiers, we aim to create neural network models provably defended from perturbations.
        Prior certification work requires strong assumptions on network structures and massive computational costs, and thus the range of their applications was limited.
        From the relationship between the Lipschitz constants and prediction margins, we present a computationally efficient calculation technique to lower-bound the size of adversarial perturbations that can deceive networks, and that is widely applicable to various complicated networks.
        Moreover, we propose an efficient training procedure that robustifies networks and significantly improves the provably guarded areas around data points.
        In experimental evaluations, our method showed its ability to provide a non-trivial guarantee and enhance robustness for even large networks.

    \end{abstract}

  \section{Introduction}
  \label{sec:introduction}

  { 

    Deep neural networks are highly vulnerable against intentionally created small perturbations on inputs \citep{Intriguing}, called adversarial perturbations, which cause serious security concerns in applications such as self-driving cars.
    Adversarial perturbations in object recognition systems have been intensively studied ~\citep{Intriguing,FGSM,CW}, and we mainly target the object recognition systems.
  }

    One approach to defend from adversarial perturbations is to mask gradients.
    Defensive distillation \citep{DefensiveDistillation}, which distills networks into themselves, is one of the most prominent methods.
    However, \citet{CW} showed that we can create adversarial perturbations that deceive networks trained with defensive distillation.
    Input transformations and detections \citep{FeatureSqueezing, InputTransform} are some other defense strategies, although we can bypass them \citep{Bypassed}.
    Adversarial training \citep{FGSM, AtScale, Madry}, which injects adversarially perturbed data into training data, is a promising approach.
    However, there is a risk of overfitting to attacks \citep{AtScale, Ensemble}.
    Many other heuristics have been developed to make neural networks insensitive against small perturbations on inputs.
    However, recent work has repeatedly succeeded to create adversarial perturbations for networks protected with heuristics in the literature \citep{Survey2}.
    For instance, \citet{Obfuscated} reported that many ICLR 2018 defense papers did not adequately protect networks soon after the announcement of their acceptance.
    This indicates that even protected networks can be unexpectedly vulnerable, which is a crucial problem for this specific line of research because the primary concern of these studies is security threats.

  { 

    The literature indicates the difficulty of defense evaluations.
    Thus, our goal is to ensure the lower bounds on the size of adversarial perturbations that can deceive networks for each input.
    Many existing approaches, which we cover in Sec.~\ref{sec:related_work}, are applicable only for special-structured small networks.
    On the other hand, common networks used in evaluations of defense methods are wide, which makes prior methods computationally intractable and complicated, which makes some prior methods inapplicable.
    This work tackled this problem, and we provide a widely applicable, yet, highly scalable method that ensures large guarded areas for a wide range of network structures.

  }

  {
    The existence of adversarial perturbations indicates that the slope of the loss landscape around data points is large,
    and we aim to bound the slope.
    An intuitive way to measure the slope is to calculate the size of the gradient of a loss with respect to an input.
    However, it is known to provide a false sense of security \citep{Ensemble, CW, Obfuscated}.
    Thus, we require upper-bounds of the gradients.
    The next candidate is to calculate a local Lipschitz constant, that is the maximum size of the gradients around each data point.
    Even though this can provide certification, calculating the Lipschitz constant is computationally hard.
    We can obtain it in only small networks or get its approximation, which cannot provide certification \citep{FormalGuarantee, Extreme, ExtremeRebuttal}.
    A coarser but available alternative is to calculate the global Lipschitz constant.
   However, prior work could provide only magnitudes of smaller certifications compared to the usual discretization of images even for small networks \citep{Intriguing, LowerBounds}.
    We show that we can overcome such looseness with our improved and unified bounds and a developed training procedure.
    The training procedure is more general and effective than previous approaches~\citep{Parseval, SpectralNorm}.
    We empirically observed that the training procedure also improves robustness against current attack methods.
  }

  \section{Related work}
  \label{sec:related_work}

  {
    In this section, we review prior work to provide certifications for networks.
    One of the popular approaches is restricting discussion to networks using ReLU \citep{ReLU} exclusively as their activation functions and reducing the verification problem to some other well-studied problems.
    \citet{LP} encoded networks to linear programs, \citet{ReluPlex, TowardsProving} reduced the problem to Satisfiability Modulo Theory, and \citet{SDP} encoded networks to semidefinite programs.
    However, these formulations demand prohibitive computational costs and their applications are limited to only small networks.
    As a relatively tractable method, \citet{OuterPolytope} has bounded the influence of $\ell_\infty$-norm bounded perturbations using convex outer-polytopes.
    However, it is still hard to scale this method to deep or wide networks.
    Another approach is assuming smoothness of networks and losses.
    \citet{FormalGuarantee} focused on local Lipschitz constants of neural networks around each input.
    However, the guarantee is provided only for networks with one hidden layer.
    \citet{DistributionalRobustness} proposed a certifiable procedure of adversarial training.
    However, smoothness constants, which their certification requires, are usually unavailable or infinite.
    As a concurrent work, \citet{Reachability} proposed another algorithm to certify robustness with more scalable manner than previous approaches.
    We note that our algorithm is still significantly faster.
  }

    \section{Problem formulation}
    \label{sec:formulation}
    We define the threat model, our defense goal, and basic terminologies.

    \paragraph{Threat model:}
    Let $X$ be a data point from data distribution $D$ and its true label be $t_X \in \{1,\dots,K\}$ where $K$ is the number of classes.
    Attackers create a new data point similar to $X$ which deceives defenders' classifiers.
    In this paper, we consider the  $\ell_2$-norm as a similarity measure between data points because it is one of the most common metrics \citep{DeepFool, CW}.

    Let $c$ be a positive constant and $F$ be a classifier.
    We assume that the output of $F$ is a vector $F(X)$ and the classifier predicts the label with $\argmax_{i\in\{1,\dots,K\}}\{F(X)_i\}$, where $F(X)_i$ denotes the $i$-th element of $F(X)$.
    Now, we define adversarial perturbation $\epsilon_{F, X}$ as follows.
    \begin{equation}
        \epsilon_{F, X}\in\left\{\epsilon\middle|\|\epsilon\|_2 < c\ \land\  t_X \neq \argmax_{i\in\{1,\dots,K\}}\left\{F\left(X + \epsilon\right)_i\right\} \right\}.\nonumber
    \end{equation}
    \paragraph{Defense goal:}
    We define a guarded area for a network $F$ and a data point $X$ as a hypersphere with a radius $c$ that satisfies the following condition:
    \begin{equation}
        \label{certification}
        \forall\epsilon, \left(\|\epsilon\|_2 < c \ \Rightarrow\  t_X = \argmax_{i\in\{1,\dots,K\}}\left\{F\left(X+\epsilon\right)_i\right\}\right).
    \end{equation}
    This condition~\eqref{certification} is always satisfied when $c=0$.
    Our goal is to ensure that neural networks have larger guarded areas for data points in data distribution.

    \section{Calculation and enlargement of guarded area}
    \label{sec:overview}
    In this section, we first describe basic concepts for calculating the provably guarded area defined in \Sec~\ref{sec:formulation}.
    Next, we outline our training procedure to enlarge the guarded area.

    \subsection{Lipschitz constant and guarded area}
    \label{subsec:absence}
    We explain how to calculate the guarded area using the Lipschitz constant.
    If $L_F$ bounds the Lipschitz constant of neural network $F$, we have the following from the definition of the Lipschitz constant:
    \begin{equation}
        \|F(X) - F(X + \epsilon)\|_2 \leq L_{F}\|\epsilon\|_2. \nonumber
    \end{equation}
    Note that if the last layer of $F$ is softmax, we only need to consider the subnetwork before the softmax layer.
    We introduce the notion of prediction margin $M_{F, X}$:
    \begin{equation}
        M_{F, X} := F(X)_{t_X} - \max_{i\neq t_X}\{F(X)_{i}\}. \nonumber
    \end{equation}
    This margin has been studied in relationship to generalization bounds \citep{PACBayesMargins,SpectrallyNoramlizedMargin,PacBayesSpectral}.
    Using the prediction margin, we can prove the following proposition holds.
    \begin{prop}
      \label{prop1}
      \begin{equation}
        \label{eq:prop1}
        (M_{F, X} \geq \sqrt{2}L_{F}\|\epsilon\|_2) \Rightarrow (M_{F, X + \epsilon} \geq 0).
      \end{equation}
    \end{prop}
    The details of the proof are in \appendixing{~\ref{sec:appendix-prop1}}.
    Thus, perturbations smaller than $M_{F, X} / \left(\sqrt{2}L_{F}\right)$ cannot deceive the network $F$ for a data point $X$.
    Proposition~\ref{prop1} sees network $F$ as a function with a multidimensional output.
    This connects the Lipschitz constant of a network, which has been discussed in \citet{Intriguing} and \citet{Parseval}, with the absence of adversarial perturbations.
    If we cast the problem to a set of functions with a one-dimensional output, we can obtain a variant of Prop.~\ref{prop1}.
    Assume that the last layer before softmax in $F$ is a fully-connected layer and $w_i$ is the $i$-th raw of its weight matrix.
    Let $L_{\mathrm{sub}}$ be a Lipschitz constant of a sub-network of $F$ before the last fully-connected layer.
    We obtain the following proposition directly from the definition of the Lipschitz constant ~\citep{FormalGuarantee, Extreme}.
    \begin{prop}
      \label{prop2}
      \begin{equation}
        \label{eq:prop2}
        (\forall i, (F_{t_X} - F_{i} \geq L_{\mathrm{sub}}\|w_{t_X} - w_{i}\|_2\|\epsilon\|_2)) \Rightarrow (M_{F, X + \epsilon} \geq 0).
      \end{equation}
    \end{prop}
    We can use either Prop.~\ref{prop1} or Prop.~\ref{prop2} for the certification.
    Calculations of the Lipschitz constant, which is not straightforward in large and complex networks, will be explained in Sec.~\ref{sec:calculation}.

    \subsection{Guarded area enlargement}
    \label{subsec:training}
    To ensure non-trivial guarded areas, we propose a training procedure that enlarges the provably guarded area.

    \paragraph{Lipschitz-margin training:}
    To encourage conditions Eq.\eqref{eq:prop1} or Eq.\eqref{eq:prop2} to be satisfied with the training data, we convert them into losses.
    We take Eq.\eqref{eq:prop1} as an example.
    To make Eq.\eqref{eq:prop1} satisfied for perturbations with $\ell_2$-norm larger than $c$, we require the following condition.
    \begin{equation}
      \forall i\neq t_X, (F_{t_X} \geq F_i + \sqrt{2}cL_F).
    \end{equation}
    Thus, we add $\sqrt{2}cL_{F}$ to all elements in logits except for the index corresponding to $t_X$.
    In training, we calculate an estimation of the upper bound of $L_F$ with a computationally efficient and differentiable way and use it instead of $L_F$.
    Hyperparameter $c$ is specified by users.
    We call this training procedure \emph{Lipschitz-margin training} (LMT).
    The algorithm is provided in Figure~\ref{alg:lmt}.
    Using Eq.(\ref{eq:prop2}) instead of Eq.(\ref{eq:prop1}) is straightforward.
    Small additional techniques to make LMT more stable is given in \appendixing{~\ref{appendix-stabilize}}.

    \paragraph{Interpretation of LMT:}
    From the former paragraph, we can see that LMT maximizes the number of training data points that have larger guarded areas than $c$, as long as the original training procedure maximizes the number of them that are correctly classified.
    We experimentally evaluate its generalization to test data in Sec.~\ref{sec:experiments}.
    The hyperparameter $c$ is easy to interpret and easy to tune.
    The larger $c$ we specify, the stronger invariant property the trained network will have.
    However, this does not mean that the trained network always has high accuracy against noisy examples.
    To see this, consider the case where $c$ is extremely large.
    In such a case, constant functions become an optimal solution.
    We can interpret LMT as an interpolation between the original function, which is highly expressive but extremely non-smooth, and constant functions, which are robust and smooth.

    \paragraph{Computational costs:}
    A main computational overhead of LMT is the calculation of the Lipschitz constant.
    We show in Sec.~\ref{sec:calculation} that its computational cost is almost the same as increasing the batch size by one.
    Since we typically have tens or hundreds of samples in a mini-batch, this cost is negligible.

    \begin{figure}[b]
        \begin{minipage}{.47\hsize}
            \begin{center}
                \begin{algorithm}[H]
                    \SetStartEndCondition{ }{}{}%
                    \SetKwProg{Fn}{def}{\string:}{}
                    \SetKwFunction{Range}{range}
                    \SetKw{KwTo}{in}\SetKwFor{For}{for}{\string:}{}%
                    \SetKwIF{If}{ElseIf}{Else}{if}{:}{elif}{else:}{}%
                    \SetKwFor{While}{while}{:}{fintq}%
                    \SetKwFor{ForEach}{foreach}{:}{fintq}%
                    \AlgoDontDisplayBlockMarkers\SetAlgoNoEnd\SetAlgoNoLine%
                    \SetKwFunction{CalcGrad}{CalcGrad}
                    \SetKwFunction{CalcLoss}{CalcLoss}
                    \SetKwFunction{Forward}{Forward}
                    \SetKwFunction{CalcLipschitzConst}{CalcLipschitzConst}
                    \SetKwFunction{SoftmaxIfNecessary}{SoftmaxIfNecessary}
                    \SetKwFunction{Encode}{Encode}
                    \SetKwFunction{CommunicateAndUpdate}{CommunicateAndUpdate}
                    \SetKwInOut{hyperparam}{hyperparam}
                    \SetKwInOut{input}{input}
                    \hyperparam{$c :$ required robustness}
                    \input{$X :$ image, $t_X$: label of $X$}
                    $y \leftarrow $ Forward($X$)\;
                    $L \leftarrow $CalcLipschitzConst()\;
                    \ForEach{index $i$}{
                      \If{$i \neq t_X$}{
                          $y_i \mathrel{+}= \sqrt{2}Lc$\;
                        }
                    }
                    $p \leftarrow $ SoftmaxIfNecessary($y$)\;
                    $\ell \leftarrow $ CalcLoss($p, t_X$)\;
                    \caption{Lipschitz-margin training}
                \end{algorithm}
            \end{center}
            \caption{
              Lipschitz-margin training algorithm when we use Prop.~\ref{prop1}.
            }
            \label{alg:lmt}
        \end{minipage}
        \begin{minipage}{.05\hsize}
            \begin{center}
            \end{center}
        \end{minipage}
        \begin{minipage}{.47\hsize}
            \begin{center}
                \begin{algorithm}[H]
                    \SetStartEndCondition{ }{}{}%
                    \SetKwProg{Fn}{def}{\string:}{}
                    \SetKwFunction{Range}{range}
                    \SetKw{KwTo}{in}\SetKwFor{For}{for}{\string:}{}%
                    \SetKwIF{If}{ElseIf}{Else}{if}{:}{elif}{else:}{}%
                    \SetKwFor{While}{while}{:}{fintq}%
                    \SetKwFor{ForEach}{foreach}{:}{fintq}%
                    \AlgoDontDisplayBlockMarkers\SetAlgoNoEnd\SetAlgoNoLine%
                    \SetKwFunction{CalcGrad}{CalcGrad}
                    \SetKwFunction{Sign}{Sign}
                    \SetKwFunction{CalcGrad}{CalcGrad}
                    \SetKwFunction{Encode}{Encode}
                    \SetKwFunction{CalcLipschitzConst}{CalcLipschitzConst}
                    \SetKwFunction{CalcLoss}{CalcLoss}
                    \SetKwInOut{hyperparam}{hyperparam}
                    \SetKwInOut{input}{input}
                    \SetKwInOut{target}{target}
                    \input{$u :$ array at previous iteration}
                    \target{$f :$ linear function}
                    $u \leftarrow u / \|u\|_2$\;
                    // $\sigma$ is an approximated spectral norm\;
                    $\sigma \leftarrow \|f(u)\|_2$\;
                    $L \leftarrow $ CalcLipschitzConst($\sigma$)\;
                    $\ell \leftarrow $ CalcLoss($L$)\;
                    $u \leftarrow \frac{\partial \ell}{\partial u}$\;
                    \caption{Calculation of operator norm}
                \end{algorithm}
            \end{center}
            \caption{
              Calculation of the spectral norm of linear components at training time.
            }
            \label{alg:power}
        \end{minipage}
    \end{figure}

    \section{Calculation of the Lipschitz constant}
    \label{sec:calculation}
    In this section, we first describe a method to calculate upper bounds of the Lipschitz constant.
    We bound the Lipschitz constant of each component and recursively calculate the overall bound.
    The concept is from \citet{Intriguing}.
    While prior work required separate analysis for slightly different components \citep{Intriguing, LowerBounds, Parseval, Reachability}, we provide a more unified analysis.
    Furthermore, we provide a fast calculation algorithm for both the upper bounds and their differentiable approximation.

    \subsection{Composition, addition, and concatenation}
    \label{subsubsec:composition}
    We describe the relationships between the Lipschitz constants and some functionals which frequently appears in deep neural networks: composition, addition, and concatenation.
    Let $f$ and $g$ be functions with Lipschitz constants  bounded by $L_1$ and $L_2$, respectively.
    The Lipschitz constant of output for each functional is bounded as follows:
    \begin{align}
        \small
        \text{composition} \hspace{0.1in} f \circ g &: L_1 \cdot L_2, \hspace{0.1in} 
        \text{addition} \hspace{0.1in} f + g &: L_1 + L_2, \hspace{0.1in} 
        \text{concatenation} \hspace{0.1in} \left(f, g\right) &: \sqrt{L_1^2 + L_2^2}. \nonumber
    \end{align}

    \subsection{Major Components}
    \label{subsubsec:layers}
    We describe bounds of the Lipschitz constants of major layers commonly used in image recognition tasks.
    We note that we can ignore any bias parameters because they do not change the Lipschitz constants of each layer.

    \paragraph{Linear layers in general:}
    Fully-connected, convolutional and normalization layers are typically linear operations at inference time.
    For instance, batch-normalization is a multiplication of a diagonal matrix whose $i$-th element is $\gamma_i/\sqrt{\sigma_i^2 + \epsilon}$, where $\gamma_i, \sigma_i^2, \epsilon$ are
    a scaling parameter, running average of variance, and a constant, respectively.
    Since the composition of linear operators is also linear, we can jointly calculate the Lipschitz constant of some common pairs of layers such as convolution + batch-normalization.
    By using the following theorem, we proposed a more unified algorithm than \citet{SpectralNorm}.
    \begin{theorem}
    \label{thm-itergrad}
    Let $\phi$ be a linear operator from $\mathbb{R}^n$ to $\mathbb{R}^m$, where $n<\infty$ and $m<\infty$.
    We initialize a vector $u\in \mathbb{R}^{n}$ from a Gaussian with zero mean and unit variance.
    When we iteratively apply the following update formula,
    the $\ell_2$-norm of $u$ converges to the square of the operator norm of $\phi$ in terms of $\ell_2$-norm, almost surely.
    \begin{align}
    u \leftarrow u / \|u\|_2, \hspace{0.2in}
    v \leftarrow \phi(u), \hspace{0.2in}
    u \leftarrow \frac{1}{2}\frac{\partial \|v\|_2^2}{\partial u}. \nonumber
    \end{align}
    \end{theorem}
    The proof is found in \appendixing{~\ref{subsec:appendix-thm-itergrad}}.
    The algorithm for training time is provided in Figure~\ref{alg:power}.
    At training time, we need only one iteration of the above update formula as with \citet{SpectralNorm}.
    Note that for estimation of the operator norm for a forward path, we do not require to use gradients.
    In a convolutional layer, for instance, we do not require another convolution operation or transposed convolution.
    We only need to increase the batch size by one.
    The wide availability of our calculation method will be especially useful when more complicated linear operators than usual convolution appear in the future.
    Since we want to ensure that the calculated bound is an \emph{upper-bound} for certification, we can use the following theorem.
    \begin{theorem}
    \label{thm-iteration}
    Let $\|\phi\|_2$ and $R_k$ be an operator norm of a function $\phi$ in terms of the $\ell_2$-norm and the $\ell_2$-norm of the vector $u$ at the $k$-th iteration, where each element $u$ is initialized by a Gaussian with zero mean and unit variance.
    With probability higher than $1 - \sqrt{2/\pi}$, the error between $\|\phi\|_2^2$ and $R_k$ is smaller than $(\Delta_k + \sqrt{\Delta_k(4R_k+\Delta_k)})/2$, where $\Delta_k := (R_k - R_{k-1})n$.
    \end{theorem}
    The proof is in \appendixing{~\ref{subsec:appendix-thm-iteration}}, which is mostly from \citet{ErrorBound}.
    If we use a large batch for the power iteration, the probability becomes exponentially closer to one.
    We can also use singular value decomposition as another way for accurate calculation.
    Despite its simplicity, the obtained bound for convolutional layers is much tighter than the previous results in \citet{LowerBounds} and \citet{Parseval}, and that for normalization layers is novel.
    We numerically confirm the improvement of bounds in Sec.~\ref{sec:experiments}.

    \paragraph{Pooling and activation:}
    First, we have the following theorem.
    \begin{theorem}
    \label{thm-pooling}
    Define $f(Z) = (f_1(Z^1), f_2(Z^2), ..., f_\Lambda(Z^\Lambda))$, where $Z^\lambda \subset Z$ and $\|f_\lambda\|_2\leq L$ for all $\lambda$. Then,
    \[\|f\|_2 \leq \sqrt{n}L,\nonumber\]
    where $n := \max_j \lvert\{\lambda| x_j \in Z^\lambda\}\rvert$ and $x_j$ is the $j$-th element of $x$.
    \end{theorem}
    The proof, whose idea comes from \citet{Parseval}, is found in \appendixing{~\ref{subsec:appendix-thm-pooling}}.
    The exact form of $n$ in the pooling and convolutional layers is given in \appendixing{~\ref{subsec:appendix-n-rep}}.
    The assumption in Theorem~\ref{thm-pooling} holds for most layers of networks for image recognition tasks, including pooling layers, convolutional layers, and activation functions.
    Careful counting of $n$ leads to improved bounds on the relationship between the Lipschitz constant of a convolutional layer and the spectral norm of its reshaped kernel from the previous result ~\citep{Parseval}.
    \begin{corollary}
    \label{conv_compare}
    Let $\|\mathrm{Conv}\|_2$ be the operator norm of a convolutional layer in terms of the $\ell_2$-norm, and $\|W'\|_2$ be the spectral norm of a matrix where the kernel of the convolution is reshaped into a matrix with the same number of rows as its output channel size. Assume that the width and the height of its input before padding are larger or equal to those of the kernel. The following inequality holds. 
    \[\|W'\|_2 \leq \|\mathrm{Conv}\|_2 \leq \sqrt{n}\|W'\|_2 \nonumber,\]
    where $n$ is a constant independent of the weight matrix.
    \end{corollary}
    The proof of Corollary~\ref{conv_compare} is in \appendixing{~\ref{subsec:appendix-conv_compare}}.
    Lists of the Lipschitz constant of pooling layers and activation functions are summarized in \appendixing{~\ref{sec:appendix-pooling}}.

    \subsection{Putting them together}
    \label{subsubbsec:altogether}
    With recursive computation using the bounds described in the previous sections, we can calculate an upper bound of the Lipschitz constants of the whole network in a differentiable manner with respect to network parameters.
    At inference time, calculation of the Lipschitz constant is required only once.

    In calculations at training time, there may be some notable differences in the Lipschitz constants.
    For example, $\sigma_i$ in a batch normalization layer depends on its input.
    However, we empirically found that calculating the Lipschitz constants using the same bound as inference time effectively regularizes the Lipschitz constant.
    This lets us deal with batch-normalization layers, which prior work ignored despite its impact on the Lipschitz constant~\citep{Parseval, SpectralNorm}.

    \section{Numerical evaluations}
    \label{sec:experiments}
    In this section, we show the results of numerical evaluations.
    Since our goal is to create networks with stronger certification, we evaluated the following three points.
    \begin{enumerate}
        \item Our bounds of the Lipschitz constants are tighter than previous ones (Sec.~\ref{subsec:exp_tightness}).
        \item LMT effectively enlarges the provably guarded area (Secs.~\ref{subsec:exp_tightness} and \ref{subsec:large_network}).
        \item Our calculation technique of the guarded area and LMT are available for modern large and complex networks (Sec.~\ref{subsec:large_network}).
    \end{enumerate}
    We also evaluated the robustness of trained networks against current attacks and confirmed that LMT robustifies networks (Secs.~\ref{subsec:exp_tightness} and~\ref{subsec:large_network}).
    For calculating the Lipschitz constant and guarded area, we used Prop.~\ref{prop2}.
    Detailed experimental setups are available in \appendixing{~\ref{sec:appendix-experiment}}.
    Our codes are available at \github.

    \subsection{Tightness of bounds}
    \label{subsec:exp_tightness}
    We numerically validated improvements of bounds for each component and
    numerically analyzed the tightness of overall bounds of the Lipschitz constant.
    We also see the non-triviality of the provably guarded area.
    We used the same network and hyperparameters as \citet{OuterPolytope}.

    \paragraph{Improvement in each component:}
    We evaluated the difference of bounds in convolutional layers in networks trained using a usual training procedure and LMT.
    Figure~\ref{fig:compare-bounds} shows comparisons between the bounds in the second convolutional layer.
    It also shows the difference of bounds in pooling layers, which does not depend on training methods.
    We can confirm improvement in each bound.
    This results in significant differences in upper-bounds of the Lipschitz constants of the whole networks.

    \begin{figure}[t]
      \begin{minipage}{.33\hsize}
        \begin{center}
          \includegraphics[width=\linewidth]{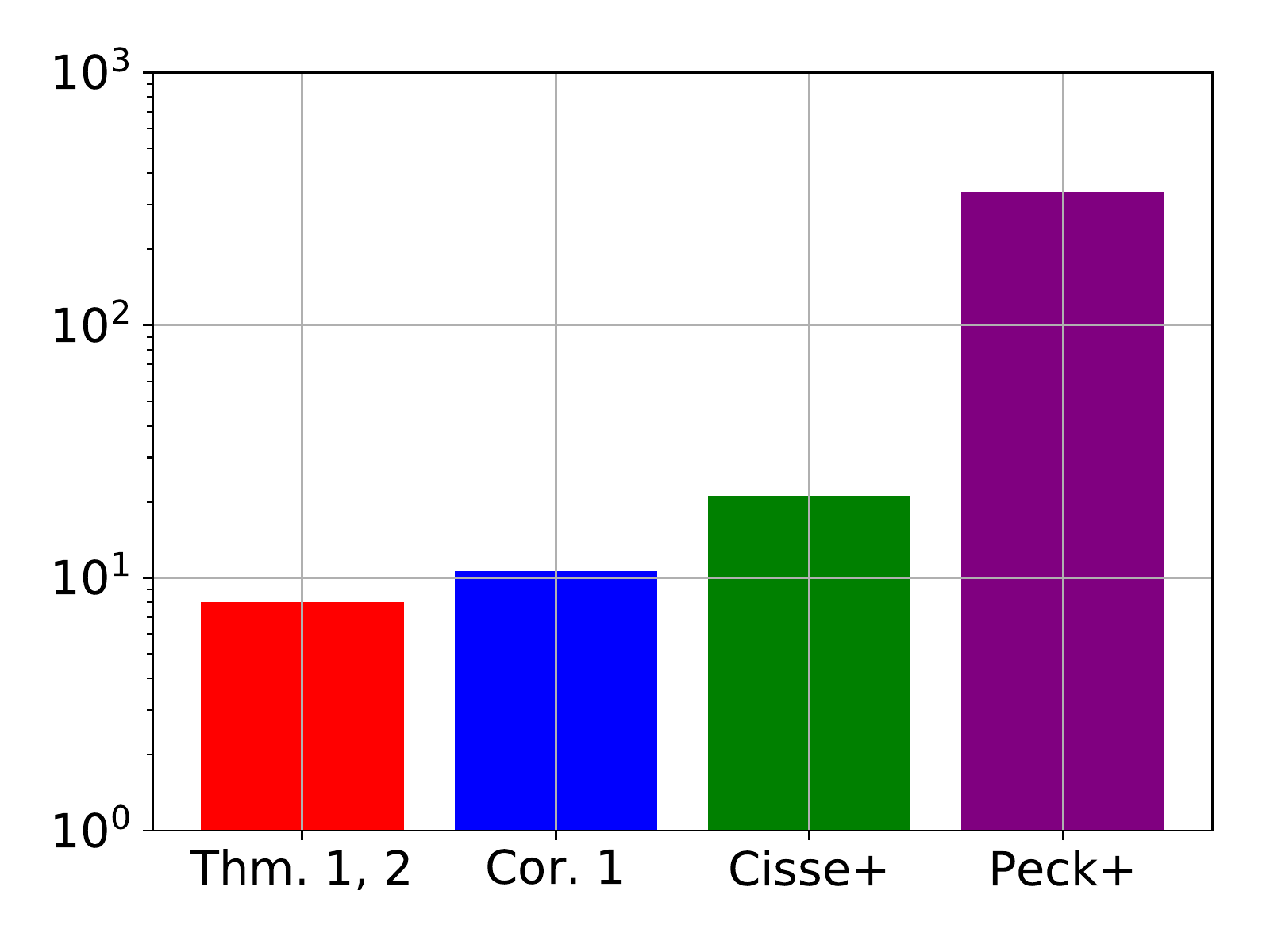}
        \end{center}
      \end{minipage}
      \begin{minipage}{.33\hsize}
        \begin{center}
          \includegraphics[width=\linewidth]{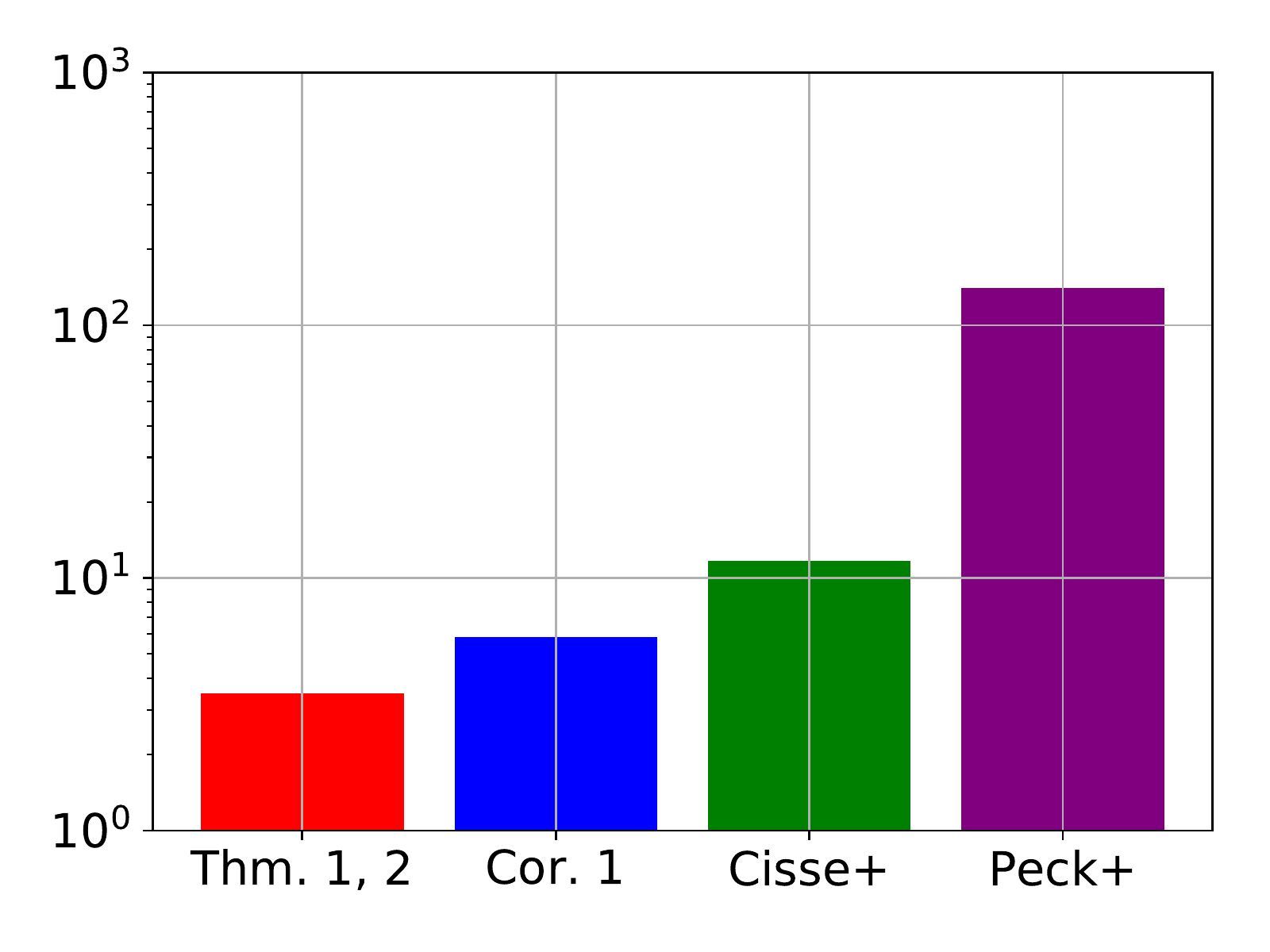}
        \end{center}
      \end{minipage}
      \begin{minipage}{.33\hsize}
        \begin{center}
          \includegraphics[width=\linewidth]{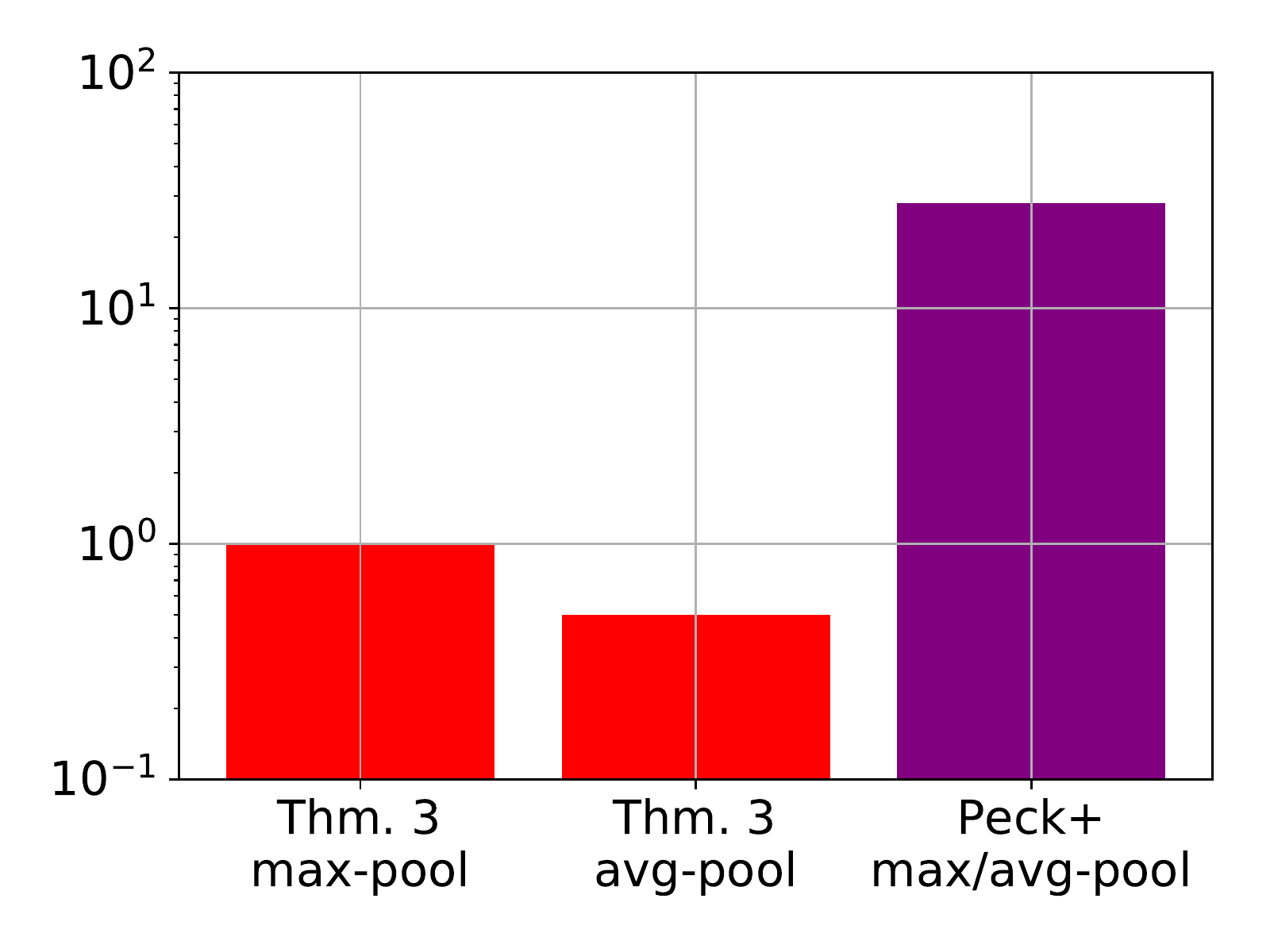}
        \end{center}
      \end{minipage}
      \caption{
        Comparison of bounds in layers.
        Left: the second convolutional layer of a naive model in Sec~\ref{sec:experiments}.
        Center: the second convolutional layer of an LMT model in Sec~\ref{sec:experiments}.
        Right: pooling layers assuming size 2 and its input size is $28\times 28$.
      }
      \label{fig:compare-bounds}
    \end{figure}

    \paragraph{Analysis of tightness:}
    Let $L$ be an upper-bound of the Lipschitz constant calculated by our method.
    Let $L_{\text{local}}, L_{\text{global}}$ be the local and global Lipschitz constants.
    Between them, we have the following relationship.
    \begin{align}
    \label{eq:ineq}
      \underbrace{\frac{\text{Margin}}{L}}_{\text{(A)}}
      \underset{\text{(\rmnum{1})}}{\leq} \underbrace{\frac{\text{Margin}}{L_{\text{global}}}}_{\text{(B)}}
      \underset{\text{(\rmnum{2})}}{\leq} \underbrace{\frac{\text{Margin}}{L_{\text{local}}}}_{\text{(C)}}
      \underset{\text{(\rmnum{3})}}{\leq} \underbrace{\text{Smallest Adversarial Perturbation}}_{\text{(D)}}
    \end{align}
    We analyzed errors in inequalities (\rmnum{1}) -- (\rmnum{3}).
    We define an error of (\rmnum{1}) as (B)$/$(A) and others in the same way.
    We used lower bounds of the local and global Lipschitz constant calculated by the maximum size of gradients found.
    A detailed procedure for the calculation is explained in \appendixing{~\ref{subsubsec:appendix-ineq}}.
    For the generation of adversarial perturbations, we used DeepFool \citep{DeepFool}.
    Note that (\rmnum{3}) does not hold because we calculated mere lower bounds of Lipschitz constants in (B) and (C).
    We analyzed inequality~\eqref{eq:ineq} in an unregularized model, an adversarially trained (AT) model with the $30$-iteration C\&W attack~\citep{CW}, and an LMT model.
    Figure ~\ref{fig:compare-robustness} shows the result.
    With an unregularized model, estimated error ratios in (\rmnum{1}) -- (\rmnum{3}) were $39.9$, $1.13$, and $1.82$ respectively.
    This shows that even if we could precisely calculate the local Lipschitz constant for each data point with possibly substantial computational costs, inequality (\rmnum{3}) becomes more than $1.8$ times looser than the size of adversarial perturbations found by DeepFool.
    In an AT model, the discrepancy became more than 2.4.
    On the other hand, in an LMT model, estimated error ratios in (\rmnum{1}) -- (\rmnum{3}) were $1.42$, $1.02$, and $1.15$ respectively.
    The overall median error between the size of found adversarial perturbations, and the provably guarded area was $1.72$.
    This shows that the trained network became smooth and Lipschitz constant based certifications became significantly tighter when we use LMT.
    This also resulted in better defense against attack.
    For reference, the median of found adversarial perturbations for an unregularized model was $0.97$, while the median of the size of the provably guarded area was $1.02$ in an LMT model.
    \begin{figure}[t]
      \begin{minipage}{.33\hsize}
        \begin{center}
          \includegraphics[width=\linewidth]{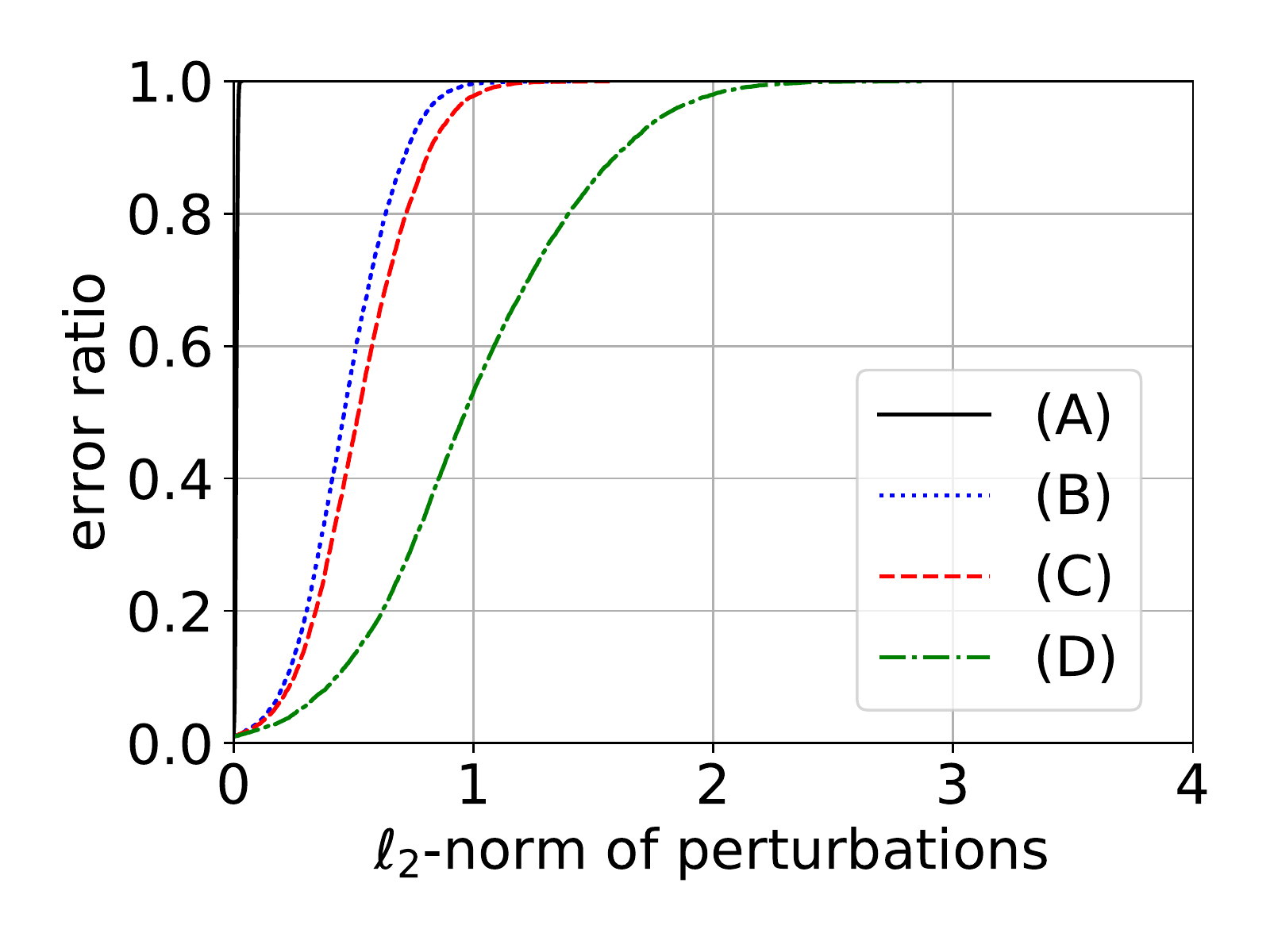}
        \end{center}
      \end{minipage}
      \begin{minipage}{.33\hsize}
        \begin{center}
          \includegraphics[width=\linewidth]{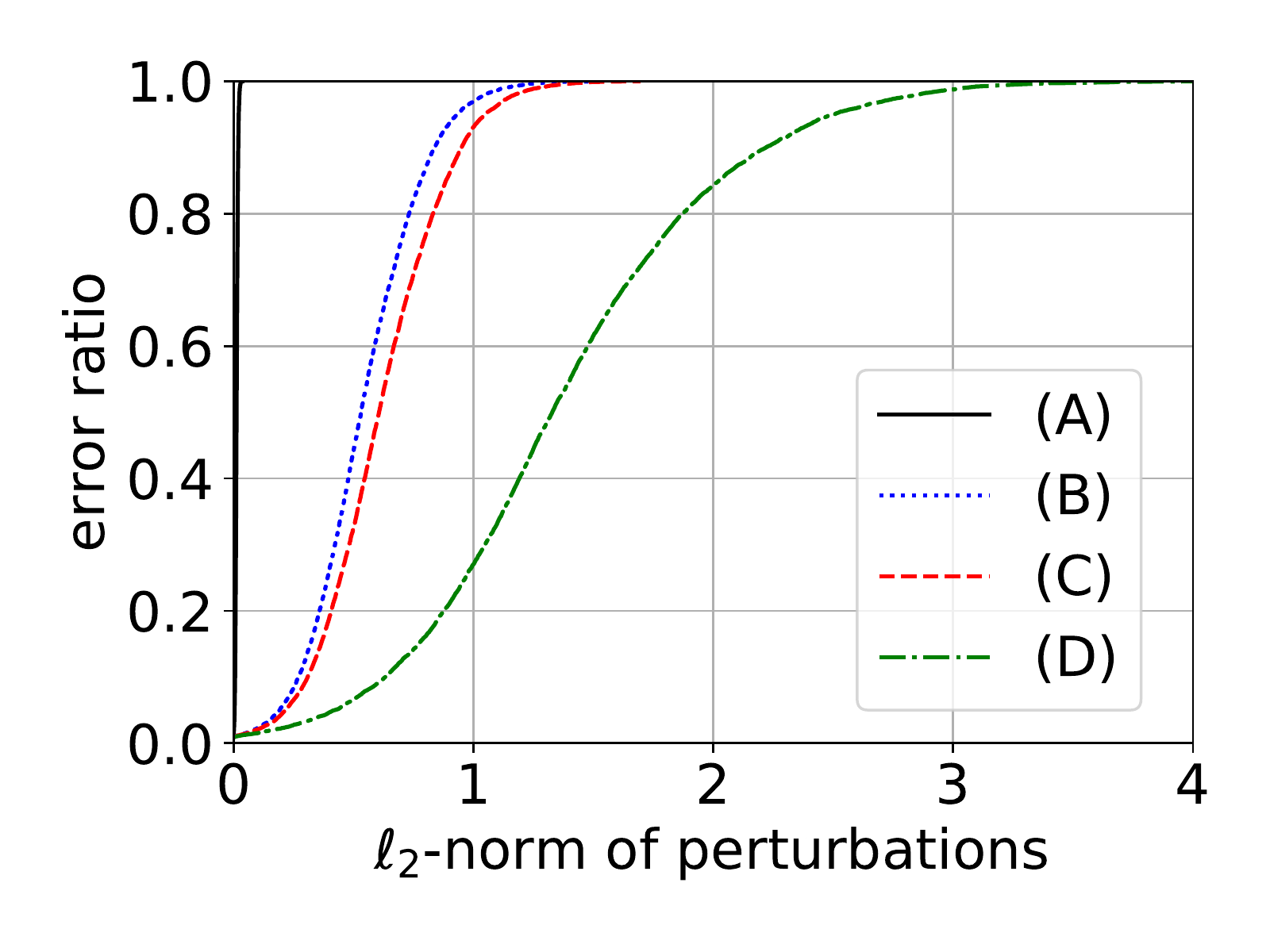}
        \end{center}
      \end{minipage}
      \begin{minipage}{.33\hsize}
        \begin{center}
          \includegraphics[width=\linewidth]{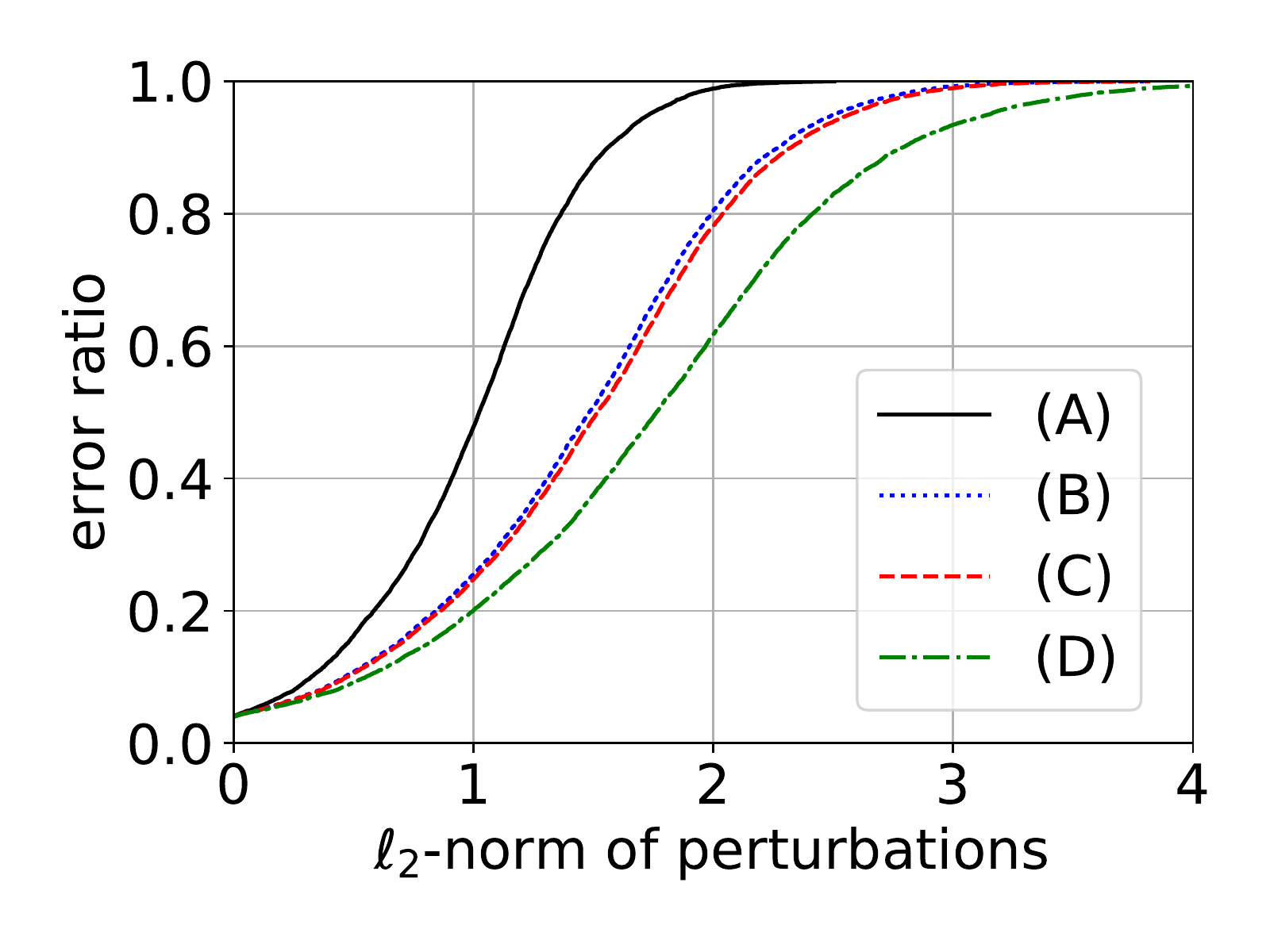}
        \end{center}
      \end{minipage}
      \caption{
        Comparison of error bounds using inequalities~\eqref{eq:ineq} with estimation.
        Each label corresponds to the value in inequality~\eqref{eq:ineq}.
        Left: naive model, Center: AT model, Right: LMT model.
      }
      \label{fig:compare-robustness}
    \end{figure}

    \paragraph{Size of provably guarded area:}
    We discuss the size of the provably guarded area, which is practically more interesting than tightness.
    While our algorithm has clear advantages on computational costs and broad applicability over prior work, guarded areas that our algorithm ensured were non-trivially large.
    In a naive model, the median of the size of perturbations we could certify invariance was $0.012$.
    This means changing several pixels by one in usual $0$--$255$ scale cannot change their prediction.
    Even though this result is not so tight as seen in the previous paragraph, this is significantly larger than prior computationally cheap algorithm proposed by \citet{LowerBounds}.
    The more impressive result was obtained in models trained with LMT, and the median of the guarded area was $1.02$.
    This corresponds to $0.036$ in the $\ell_\infty$ norm.
    \citet{OuterPolytope}, which used the same network and hyperparameters as ours, reported that they could defend from perturbations with its $\ell_\infty$-norm bounded by $0.1$ for more than $94\%$ examples.
    Thus, in the $\ell_\infty$-norm, our work is inferior, if we ignore their limited applicability and massive computational demands.
    However, our algorithm mainly targets the $\ell_2$-norm, and in that sense, the guarded area is significantly larger.
    Moreover, for more than half of the test data, we could ensure that there are no one-pixel attacks \citep{OnePixel}.
    To confirm the non-triviality of the obtained certification, we have some examples of provably guarded images in Figure~\ref{fig:perturb}.
        \begin{figure}[t]
      \begin{minipage}{.075\hsize}
        \begin{center}
          \includegraphics[width=\linewidth]{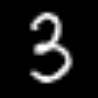}
        \end{center}
      \end{minipage}
      \begin{minipage}{.01\hsize}
        \begin{center}
        \end{center}
      \end{minipage}
      \begin{minipage}{.075\hsize}
        \begin{center}
          \includegraphics[width=\linewidth]{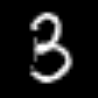}
        \end{center}
      \end{minipage}
      \begin{minipage}{.02\hsize}
        \begin{center}
        \end{center}
      \end{minipage}
      \begin{minipage}{.075\hsize}
        \begin{center}
          \includegraphics[width=\linewidth]{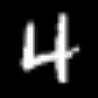}
        \end{center}
      \end{minipage}
      \begin{minipage}{.01\hsize}
        \begin{center}
        \end{center}
      \end{minipage}
      \begin{minipage}{.075\hsize}
        \begin{center}
          \includegraphics[width=\linewidth]{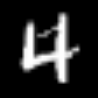}
        \end{center}
      \end{minipage}
            \begin{minipage}{.02\hsize}
        \begin{center}
        \end{center}
      \end{minipage}
      \begin{minipage}{.075\hsize}
        \begin{center}
          \includegraphics[width=\linewidth]{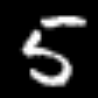}
        \end{center}
      \end{minipage}
      \begin{minipage}{.01\hsize}
        \begin{center}
        \end{center}
      \end{minipage}
      \begin{minipage}{.075\hsize}
        \begin{center}
          \includegraphics[width=\linewidth]{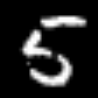}
        \end{center}
      \end{minipage}
            \begin{minipage}{.02\hsize}
        \begin{center}
        \end{center}
      \end{minipage}
      \begin{minipage}{.075\hsize}
        \begin{center}
          \includegraphics[width=\linewidth]{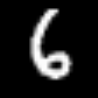}
        \end{center}
      \end{minipage}
      \begin{minipage}{.01\hsize}
        \begin{center}
        \end{center}
      \end{minipage}
      \begin{minipage}{.075\hsize}
        \begin{center}
          \includegraphics[width=\linewidth]{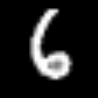}
        \end{center}
      \end{minipage}
            \begin{minipage}{.02\hsize}
        \begin{center}
        \end{center}
      \end{minipage}
      \begin{minipage}{.075\hsize}
        \begin{center}
          \includegraphics[width=\linewidth]{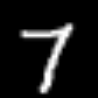}
        \end{center}
      \end{minipage}
      \begin{minipage}{.01\hsize}
        \begin{center}
        \end{center}
      \end{minipage}
      \begin{minipage}{.075\hsize}
        \begin{center}
          \includegraphics[width=\linewidth]{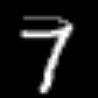}
        \end{center}
      \end{minipage}
      \caption{
        Examples of pairs of an original image and an artificially perturbed image
        which LMT model was ensured not to make wrong predictions.
        The differences between images are large and visually perceptible.
        On the basis of Proposition~\ref{prop2},
        any patterns of perturbations with the same or smaller magnitudes could not  deceive the network trained with LMT.
      }
      \label{fig:perturb}
    \end{figure}

    \subsection{Scalability test}
    \label{subsec:large_network}
    We evaluated our method with a larger and more complex network to confirm its broad applicability and scalability.
    We used $16$-layered wide residual networks \citep{WideResNet} with width factor $4$ on the SVHN dataset \citep{SVHN} following \citet{Parseval}.
    To the best of our knowledge, this is the largest network concerned with certification.
    We compared LMT with a naive counterpart, which uses weight decay, spectral norm regularization \citep{SpectralNorm}, and Parseval networks.

    \paragraph{Size of provably guarded area:}
    For a model trained with LMT, we could ensure larger guarded areas than $0.029$ for more than half of test data.
    This order of certification was only provided for small networks in prior work.
    In models trained with other methods, we could not provide such strong certification.
    There are mainly two differences between LMT and other methods.
    First, LMT enlarges prediction margins.
    Second, LMT regularizes batch-normalization layers, while in other methods, batch-normalization layers cancel the regularization on weight matrices and kernel of convolutional layers.
    We also conducted additional experiments to provide further certification for the network.
    First, we replaced convolution with kernel size 1 and stride 2 with average-pooling with size 2 and convolution with kernel size 1.
    Then, we used LMT with $c=0.1$.
    As a result, while the accuracy dropped to $86\%$, the median size of the provably guarded areas was larger than $0.08$.
    This corresponds to that changing $400$ elements of input by $\pm 1$ in usual image scales ($0$--$255$) cannot cause error over $50\%$ for the trained network.
    These certifications are non-trivial, and to the best of our knowledge, these are the best certification provided for this large network.

    \paragraph{Robustness against attack:}
    We evaluated the robustness of trained networks against adversarial perturbations created by the current attacks.
    We used C\&W attack~\citep{CW} with 100 iterations and no random restart for evaluation.
    Table~\ref{table:accuracy} summarizes the results.
    While LMT slightly dropped its accuracy, it largely improved robustness compared to other regularization based techniques.
    Since these techniques are independent of other techniques such as adversarial training or input transformations, further robustness will be expected when LMT is combined with them.

    \begin{table}
        \caption{
        Accuracy of trained wide residual networks on SVHN against C\&W  attack.
        }
        \label{table:accuracy}
        \begin{center}
            \begin{tabular}{ c c c c c c }
                \toprule
                &  & \multicolumn{3}{c}{\textbf{Size of perturbations}} \\
                \cline{3-5}
                & \textbf{Clean} & $\mathbf{0.2}$ & $\mathbf{0.5}$ & $\mathbf{1.0}$\\
                \midrule
                weight decay & $\mathbf{98.31}$ & $72.38$ & $20.98$ & $2.02$ \\
                Parseval network & $98.30$ & $71.35$ & $17.92$ & $0.94$ \\
                spectral norm regularization & $98.27$ & $73.66$ & $20.35$ & $1.39$ \\
                \textbf{LMT} & $96.38$ & $\mathbf{86.90}$ & $\mathbf{55.11}$ & $\mathbf{17.69}$ \\
                \bottomrule
            \end{tabular}
        \end{center}
    \end{table}

    \section{Conclusion}
    \label{sec:conclusion}
    To ensure perturbation invariance of a broad range of networks with a computationally efficient procedure,
    we achieved the following.
    \begin{enumerate}
        \setlength{\parskip}{0.02in}
        \setlength{\parsep}{0.02in}
        \setlength{\itemsep}{0.02in}
        \item We offered general and tighter spectral bounds for each component of neural networks.
        \item We introduced general and fast calculation algorithm for the upper bound of operator norms and its differentiable approximation.
        \item We proposed a training algorithm which effectively constrains networks to be smooth, and achieves better certification and robustness against attacks.
        \item We successfully provided non-trivial certification for small to large networks with negligible computational costs.
    \end{enumerate}
    We believe that this work will serve as an essential step towards both certifiable and robust deep learning models.
    Applying developed techniques to other Lipschitz-concerned domains such as training of GAN or training with noisy labels is future work.

    \clearpage
    
    \section*{Acknowledgement}
      Authors appreciate Takeru Miyato for valuable feedback.
      YT was supported by Toyota/Dwango AI scholarship.
      IS was supported by KAKENHI 17H04693.
      MS was supported by KAKENHI 17H00757.

    \bibliography{nips2018}
    \bibliographystyle{plainnat}

    \clearpage
    \appendix
    
  \section{Proof of Proposition~\ref{prop1}}
  \label{sec:appendix-prop1}
  We prove Prop.~\ref{prop1} in Sec.~\ref{subsec:absence}.
  Let us consider a classifier with Lipschitz constant $L$.
  Let $F(X)$ be an output vector of the classifier for a data point $X$.

  The statement to prove is the following:
  \begin{align}
    F(X)_{t_X} - \max_{i\neq t_X} F(X)_i \geq \sqrt{2}L\|\epsilon\|_2
    \Rightarrow F(X + \epsilon)_{t_X} - \max_{i\neq t_X} F(X + \epsilon)_i \geq 0 \label{eq:statement}.
  \end{align}
  If we prove the following, it suffices:
  \begin{align}
    F(X + \epsilon)_{t_X} - \max_{i\neq t_X} F(X + \epsilon)_i
    \geq F(X)_{t_X} - \max_{i\neq t_X} F(X)_i - \sqrt{2}L\|\epsilon\|_2. \label{eq:sufficient-condition}
  \end{align}

  Before proving inequality (\ref{eq:sufficient-condition}), we have the following lemma.
  \begin{lemma}
    For real vectors $x$ and $y$, the following inequality holds:
    \[
      \lvert\max_{i\neq t_X} x_i - \max_{i\neq t_X} y_i\rvert
      \leq \max_{i\neq t_X} \lvert x_i - y_i\rvert.
    \]
  \end{lemma}
  \begin{proof}
    W.l.o.g. we assume $\max_{i\neq t_X} x_i \geq \max_{i\neq t_X} y_i$.
    Let $j$ be $\argmax_{i\neq{t_X}}x_i$. Then,
    \begin{align*}
      \lvert\max_{i\neq t_X} x_i - \max_{i\neq t_X} y_i\rvert
      &= \max_{i\neq t_X} x_i - \max_{i\neq t_X} y_i \\
      &= x_j - \max_{i\neq t_X} y_i \\
      &\leq x_j - y_j \\
      &\leq \max_{i\neq t_X}\lvert x_i - y_i \rvert \hfill
    \end{align*}
  \end{proof}

  Now, we can prove the inequality (\ref{eq:sufficient-condition}).
  \begin{align}
    & F(X + \epsilon)_{t_X} - \max_{i\neq t_X} F(X + \epsilon)_i \nonumber \\
    = & F(X)_{t_X} - \max_{i\neq t_X} F(X)_i + \left(F(X+\epsilon)_{t_X} - F(X)_{t_X}\right)
    - \left(\max_{i\neq t_X}F(X+\epsilon)_i - \max_{i\neq t_X}F(X)_{i}\right) \nonumber \\
    \geq & F(X)_{t_X} - \max_{i\neq t_X} F(X)_i - \lvert F(X+\epsilon)_{t_X} - F(X)_{t_X}\rvert
    - \lvert\max_{i\neq t_X}F(X+\epsilon)_i - \max_{i\neq t_X}F(X)_{i}\rvert \nonumber \\
    \geq & F(X)_{t_X} - \max_{i\neq t_X} F(X)_i - \lvert F(X+\epsilon)_{t_X} - F(X)_{t_X}\rvert
    - \max_{i\neq t_X}\lvert F(X+\epsilon)_i - F(X)_{i}\rvert \nonumber \\
    \geq & F(X)_{t_X} - \max_{i\neq t_X} F(X)_i - \max_{a_1,a_2\in \mathbb{R}}\left\{\lvert a_{1}\rvert + \lvert a_2 \rvert \middle| \sqrt{a_1^2 + a_2^2} \leq L \|\epsilon\|_2 \right\} \nonumber \\
    = & F(X)_{t_X} - \max_{i\neq t_X} F(X)_i - \sqrt{2} L\|\epsilon\|_2. \nonumber
  \end{align}
  \qed

  \section{Lipschitz constant of basic functionals}
    \label{sec:appendix-functional}
    We prove bounds described in \Sec~\ref{subsubsec:composition}.
    Let $f$ and $g$ be functions with their Lipschitz constants
    bounded with $L_1$ and $L_2$, respectively.

    \subsection{Composition of functions}
      \label{subsubsec:appendix-composition}
      \begin{align*}
        \frac{\|f(g(X_1)) - f(g(X_2))\|_2}{\|X_1 - X_0\|_2}
        &=\frac{\|f(g(X_1)) - f(g(X_2))\|_2}{\|g(X_1) - g(X_2)\|_2}
        \cdot\frac{\|g(X_1) - g(X_2)\|_2}{\|X_1 - X_0\|_2} \\
        &\leq L_1\cdot L_2.
      \end{align*}

    \subsection{Addition of functions}
      \label{subsubsec:appendix-addition}
      Using triangle inequality,
      \begin{align*}
        \frac{\|(f + g)(X_1) - (f + g)(X_2))\|_2}{\|X_1 - X_0\|_2}
        &\leq\frac{\|f(X_1) - f(X_2))\|_2 + \|g(X_1) - g(X_2)\|_2}{\|X_1 - X_0\|_2} \\
        &\leq L_1 + L_2.
      \end{align*}

    \subsection{Concatenation of functions}
      \label{subsubsec:appendix-concatenation}
      \begin{align*}
          \frac{\|(f(X_1), g(X_1)) - (f(X_0), g(X_0))\|_2}{\|X_1 - X_0\|_2}
          &=\sqrt{\frac{\|(f(X_1), g(X_1)) - (f(X_0), g(X_0))\|_2^2}{\|X_1 - X_0\|_2^2}} \\
          &=\sqrt{\frac{\|f(X_1) - f(X_0)\|_2^2 + \|g(X_1) - g(X_0)\|_2^2}{\|X_1 - X_0\|_2^2}} \\
          &\leq\sqrt{L_1^2 + L_2^2}.
      \end{align*}

  \section{Lipschitz constant of linear components}
    \label{sec:appendix-linear}
    We see the Lipschitz constant of linear components, given in Sec.~\ref{subsubsec:layers}, in more detail.
    We first prove Theorem~\ref{thm-itergrad}, and Theorem~\ref{thm-iteration}.
    Next, we focus on its calculation for normalization layers.

    \subsection{Proof of Theorem~\ref{thm-itergrad}}
      \label{subsec:appendix-thm-itergrad}
      Since there exists a matrix representation $M$ of $\phi$ and the operator norm of $\phi$ in terms of $\ell_2$-norm is equivalent to the spectral norm of $M$, considering $v = Mu$ is sufficient.
      Now, we have
      \begin{align*}
        \frac{1}{2}\frac{\partial \|v\|_2^2}{\partial u}
        &= \frac{1}{2}\frac{\partial (u^{\top}M^{\top}Mu)}{\partial u} \\
        &= M^{\top}Mu.
      \end{align*}
      Thus, recursive application of the algorithm in Theorem~\ref{thm-itergrad} is equivalent to the power iteration to $M^{\top}M$.
      Since the maximum eigen value of $M^{\top}M$ is a square of the spectral norm of $M$,  $u$ converges to the square of the spectral norm of $M$ almost surely in the algorithm.
      \qed

    \subsection{Explanation of Algorithm~\ref{alg:power}}
      We use the same notation with Algorithm~\ref{alg:power}.
      In Algorithm~\ref{alg:power}, we only care the direction of the vector $u$ because we
      normalize it at every iteration.
      We first explain that the direction of $u$
      converges to a singular vector of the largest singular value of the linear function $f$
      when $f$ is fixed.\\
      Since
      \begin{align*}
        \frac{\partial L}{\partial u}
        &= \frac{\partial L}{\partial \sigma}\cdot\frac{\partial \sigma}{\partial u} \\
        &= 2\frac{\partial L}{\partial \sigma}\cdot\frac{\partial \sigma}{\partial \sigma^2}\cdot\left(\frac{1}{2}\frac{\partial \sigma^2}{\partial u}\right)
      \end{align*}
      and $2\frac{\partial L}{\partial \sigma}\cdot\frac{\partial \sigma}{\partial \sigma^2}$ is a scalar,
      $u$ converges to the same direction with Theorem~\ref{thm-itergrad}.
      In other words, $u$ converges to the singular vector of the largest singular of $M$.

      If $u$ approximates the singular vector, then $\sigma := \|f(u)\|_2$ approximates the spectral norm of $f$.
      If $f$ changes a little per iteration, even though Algorithm~\ref{alg:power} performs only one step of the power iteration per iteration, we can keep good approximation of the spectral norm~\citep{SpectralNorm}.

    \subsection{Proof of Theorem~\ref{thm-iteration}}
      \label{subsec:appendix-thm-iteration}
      From the proof of Theorem~\ref{thm-itergrad} in Appendix~\ref{subsec:appendix-thm-itergrad}, we considers power iteration to $M^{\top}M$.
      Let $\lambda_1$ be the largest singular value of a matrix $M^{\top}M$.
      Since $M^{\top}M$ is a symmetric positive definite matrix, from Theorem 1.1 in ~\citet{ErrorBound}, we have
      \[\lambda_1 - R_k \leq \frac{1}{2}\left(\Delta_k + \sqrt{\Delta_k\left(4R_k + \Delta_k\right)}\right),\]
      where $\Delta_k$ is bounded by $\omega - 1$ from Prop. 2.2 in \citep{ErrorBound}.
      A quantity $\omega$ has the following relationship \citep{ErrorBound}:
      \[\mathrm{Pr}(\omega - 1 \geq n) \leq \sqrt{2/\pi}.\]
      Thus, the Theorem~\ref{thm-iteration} holds.
      \qed

    If we use batchsize $128$ for the algorithm and take the max of all upper bound, then the failure probability is less than $\left(2/\pi\right)^{128/2} \leq 10^{-12}$.

    \subsection{Calculation of normalization layers}
      \label{subsec:appendix-noramlization}

    \subsubsection{Example: batch-normalization}
      \label{subsubsec:appendix-normalization}
      Batch normalization applies the following function,
      \begin{equation}
          \label{eq:batch_update}
          x_i \leftarrow \gamma_i\frac{x_i - \mu_i}{\sqrt{\sigma_i^2 + \epsilon}} + \beta_i,
      \end{equation}
      where $\gamma_i$ and $\beta_i$ are learnable parameters and $\mu_i, \sigma_i$ are the mean and deviation of (mini) batch, respectively.
      Parameters and variables $\gamma_i, \beta_i, \mu_i,$ and $\sigma_i$ are constant at the inference time.
      Small constant $\epsilon$ is generally added for numerical stability.
      We can rewrite an update of \eqref{eq:batch_update} as follows:
      \begin{equation}
          x_i \leftarrow \frac{\gamma_i}{\sqrt{\sigma_i^2 + \epsilon}}x_i + \left(-\gamma_i\frac{\mu_i}{\sqrt{\sigma_i^2 + \epsilon}} + \beta_i\right). \nonumber
      \end{equation}
      Since the second term is constant in terms of input, it is independent of the Lipschitz constant.
      Thus, we consider the following update:
      \begin{equation}
          x_i \leftarrow \frac{\gamma_i}{\sqrt{\sigma_i^2 + \epsilon}}x_i. \nonumber
      \end{equation}
      The Lipschitz constant can be bounded by $\underset{i}{\max} \{|\gamma_i|/\sqrt{\sigma_i^2 + \epsilon}\}$.
      \\
      Since the opertion is linear, we can also use Algorithm~\ref{alg:power} for the calculation.
      This allows us to calculate the Lipschitz constant of batch-noramlization and precedent other linear layers jointly.
      When we apply the algorithm~\ref{alg:power} to a single batch-normalization layer,
      a numerical issue can offer. See Appendix~\ref{subsec:numerical-issue} for more details.

    \subsubsection{Other normalizations}
      In weight normalization~\citep{WeightNorm}, the same discussion applies if we replace $\sqrt{\sigma_i^2+\epsilon}$ in batch-normalization with $\|w_i\|_2$, where $w_i$ is the $i$-th row of a weight matrix.

    \subsubsection{Undesired convergence of power iteration}
      \label{subsec:numerical-issue}
      In some cases, estimation of spectral norm using power iteration can fail in training dynamics.
      For example, in batch-normalization layer, $u$ in Algorithm~\ref{alg:power} converges to some one-hot vector.
      Once $u$ converges, no matter how much other parameters change during training, $u$ stay the same.
      To avoid the problem, when we apply Algorithm~\ref{alg:power} to normalization layers, we added small perturbations on a vector $u$ in the algorithm at every iteration after its normalization.

  \section{Lipschitz constant of pooling and activation}
    \label{sec:appendix-pooling}

    \subsection{Proof of Theorem~\ref{thm-pooling}}
      \label{subsec:appendix-thm-pooling}
      First, we prove the following lemma:

      \begin{lemma}
          Let vector $X$ be a concatenation of vectors ${X}_\lambda (0\leq \lambda \leq n)$ and
          let $f$ be a function such that $f(x)$ is a concatenations of vectors $f(X_\lambda)$,
          where each $f_\lambda$ is a function with its Lipschitz constant bounded by $L$.
          Then, the Lipschitz constant of $f$ can be bounded by $L$.
      \end{lemma}

      \begin{align}
          \frac{\|f(X) - f(Y)\|_2^2}{\|X - Y\|_2^2}
          = & \frac{\underset{\lambda}{\sum} \|f_\lambda(X_\lambda) - f_\lambda(Y_\lambda)\|_2^2}{\|X - Y\|_2^2} \nonumber \\
          \leq & \frac{\underset{\lambda}{\sum} L^2\|X_\lambda - Y_\lambda\|_2^2}{\|X - Y\|_2^2} \nonumber \\
          = & L^2 \nonumber
      \end{align}
      \qed

      Since n-th time repetition is the same with n-th time concatenation, which is explained in Appendix~\ref{subsubsec:appendix-concatenation}, its Lipschitz constant is bounded by $\sqrt{n}$.
      Using Appendix~\ref{subsubsec:appendix-concatenation} and the above lemma, we obtain the bound in Theorem~\ref{thm-pooling}.
      \qed

    \subsection{Proof of Corollary~\ref{conv_compare}}
      \label{subsec:appendix-conv_compare}
      \paragraph{notation:}
      $ch_{in}, ch_{out}, h_{k}, w_{k}$: input channel size, output channel size, kernel height, kernel width.\\
      $W'$: a matrix which kernel of a convolution $W\in R^{ch_{out}\times ch_{in} \times h_k \times w_k}$ is reshaped into the size $ch_{out} \times (ch_{in} \times h_k \times w_k)$.

      \paragraph{proof:}
      The operation in a convolution layer satisfies the assumption in Theorem~\ref{thm-pooling}, where all $f_i$ are the matrix multiplication of $W'$. Thus, the right inequality holds.
      Since matrix multiplication with $W'$ is applied at least once in the convolution, the left inequality holds.
      \qed
      
      This result is similar to \citet{Parseval}, but we can provide better bounds by carefully calculating the number of repetition, given in Appendix~\ref{subsec:appendix-n-rep}.

    \subsection{Tighter bound of n-repetition in Theorem~\ref{thm-pooling}}
      \label{subsec:appendix-n-rep}
      We provide tight number of the repetition for pooling and convolutional layers here.

      \paragraph{notation:}
      $h_{in}, w_{in}$: height and width of input array.\\
      $h_{k}, w_{k}$: kernel height, kernel width.\\
      $h_s$, $w_s$: stride height, stride width.\\

      \paragraph{number of repetition:}
      \begin{equation}
          \left\lceil \frac{\min(h_k, h_{in} - h_k + 1)}{h_s} \right\rceil
          \cdot
          \left\lceil \frac{\min(w_k, w_{in} - w_k + 1)}{w_s} \right\rceil. \nonumber
      \end{equation}

      \paragraph{derivation:}
      First of all, the repetition is bounded by the size of reception field, which is $h_k \cdot w_k$.
      This is provided by \citet{Parseval}.
      Now, we extend the bound by considering the input size and stride.
      Firstly, we consider the input size after padding.
      If both the input and kernel size are $8 \times 8$, the number of repetition is obviously bounded by $1$.
      Similarly, the number of repetition can be bounded by the following:
      \begin{equation}
          \min(h_k, h_{in} - h_k + 1)\cdot\min(w_k, w_{in} - w_k + 1). \nonumber
      \end{equation}
      We can further bound the time of repetition by considering the stride as follows:
      \begin{equation}
          \left\lceil \frac{\min(h_k, h_{in} - h_k + 1)}{h_s} \right\rceil
          \cdot
          \left\lceil \frac{\min(w_k, w_{in} - w_k + 1)}{w_s} \right\rceil. \nonumber
      \end{equation}
      \qed

    \subsubsection{The Lipschitz constant of $f_i$ in Theorem~\ref{thm-pooling} for Pooling layers}
      \label{subsubsec:appendix-pooling}

      \paragraph{max-pooling:}
      Lipschitz constant of max function is bounded by one.

      \paragraph{average-pooling:}
      Before bounding the Lipschitz constant, we note that the following inequality holds for a vector $X$:
      \begin{equation}
        \left(\underset{i=1}{\overset{n}{\sum}}X_{i}\right)^2 \leq n\underset{i=1}{\overset{n}{\sum}}X_{i}^2. \nonumber
      \end{equation}
      This can be proved using
      \begin{equation}
          \frac{1}{n}\underset{i=1}{\overset{n}{\sum}}X_{i}^2 - \left(\frac{1}{n}\underset{i=1}{\overset{n}{\sum}}X_{i}\right)^2
          = \frac{1}{n}\underset{i=1}{\overset{n}{\sum}}\left(X_{i} - \frac{1}{n}\underset{i=1}{\overset{n}{\sum}}X_{i}\right)^2
          \geq 0. \nonumber
      \end{equation}
      Now, we bound the Lipschitz constant of the average function $\mathrm{Avg}(\cdot)$.
      \begin{align}
          \frac{\|\mathrm{Avg}(X) - \mathrm{Avg}(Y)\|_2}{\|X - Y\|_2}
          = & \frac{|\mathrm{Avg}(X - Y)|}{\|X - Y\|_2} \nonumber \\
          = & \frac{|\underset{i=1}{\overset{h_k \cdot w_k}{\sum}}(X_i - Y_i)|}{\|X - Y\|_2} / (h_k \cdot w_k) \nonumber \\
          \leq & \frac{1}{\sqrt{h_k \cdot w_k}}. \nonumber
      \end{align}

    \subsection{Activation functions}
      \label{subsec:appendix-activation}
      Table~\ref{table:lipschitz} lists up the Lipschitz constants of activation functions
      and other nonlinear functions commonly used in deep neural networks.
      From Theorem~\ref{thm-pooling}, we only need to consider the Lipschitz constants elementwisely.

      \begin{table}
          \caption{
          Lipschitz constants of major activation functions.
          }
          \label{table:lipschitz}
          \begin{center}
              \begin{tabular}{ c c }
                  \toprule
                  \textbf{Activation} & \textbf{Lipschitz constant}\\
                  \midrule
                  ReLU & $1$  \\
                  Leaky ReLU \citep{LeakyReLU} & $\max (1, |\alpha|)$  \\
                  sigmoid & $1/4$ \\
                  tanh & $1$ \\
                  soft plus \citep{SoftPlus} & $1$ \\
                  ELU \citep{ELU} & $\max (1, |\alpha|)$ \\
                  \bottomrule
              \end{tabular}
          \end{center}
      \end{table}

  \section{Lipschitz-Margin Training stabilization}
  \label{appendix-stabilize}
  We empirically found that applying the addition only when a prediction is correct stabilizes the training.
  Thus, in the training, we scale the addition with
  \begin{equation}
      \alpha_{F, X} := \min_{i\neq t_X}\left\{\max\left(0, \min\left(1, \frac{F(X)_{t_X} - F(X)_i}{\sqrt{2}cL_{F}}\right)\right)\right\}. \nonumber
  \end{equation}
  Even though $\alpha_{F, X}$ depends on $L_{F}$, we do not back-propagate it.

  Similarly, we observed that strong regularization at initial stage of training can make training unstable.
  Thus, we set an initial value of $c$ small and linearly increased it to the target value in first $5$ epochs as learning rate scheduling used in \citet{OneHour}.

  \section{Experimental setups}
  \label{sec:appendix-experiment}
  In this section, we describe the details of our experimental settings.

  \subsection{Experiment~\ref{subsec:exp_tightness}}
  \label{subsec:appendix-tightness}
  
  \subsubsection{Base network}

  We used the same network, optimizer and hyperparameters with \citet{OuterPolytope}.
  A network consisting of two convolutional and two fully-connected layers was used.
  Table~\ref{table:polytope-net} shows the details of its structure.

  \begin{table}
      \caption{
      Network structure used for experiment~\ref{subsec:exp_tightness}.
      For convolutional layers, output size denotes channel size of output.
      }
      \label{table:polytope-net}
      \begin{center}
          \begin{tabular}{ c c c c c }
              \toprule
              & output size & kernel & padding & stride \\
              \midrule
              convolution     & 16  & (4,4) & (1,1) & (2,2) \\
              ReLU            & -   & -     &  -    & - \\
              convolution     & 32  & (4,4) & (1,1) & (2,2) \\
              ReLU            & -   & -     &  -    & - \\
              fully-connected & 100 & -     &  -    & - \\
              ReLU            & -   & -     &  -    & - \\
              fully-connected & 10  & -     &  -    & - \\
              \bottomrule
          \end{tabular}
      \end{center}
  \end{table}
  
  \subsubsection{Hyperprameters}

  All models were trained using Adam optimizer \citep{Adam} for $20$ epochs with a batch size of $50$.
  The learning rate of Adam was set to $0.001$.
  Note that these setting is the same with ~\citet{OuterPolytope}.
  For a LMT model, we set $c=1$.
  For an AT model, we tuned hyperparemter $c$ of C\&W attack from $[0.0001, 0.001, 0.01, 0.1, 1, 10, 100, 1000, 10000]$ and chose the best one on validation data.

	\subsubsection{Estimation of inequality~\eqref{eq:ineq}}
    \label{subsubsec:appendix-ineq}

    \paragraph{(A):}
    We calculated (A) with Proposition~\ref{prop2}.
    
    \paragraph{(B):}
    We took the max of the local Lipschitz constant calculated for (C).
    
     \paragraph{(C):}
     First, we added a random perturbation which each element is sampled from a Gaussian with zero-mean and variance $v$ , where $v$ is set as a reciprocal number of the size of input dimension.
     Next, we calculated the size of a gradient with respect to the input.
     We repeated the above two for 100 times and used the maximum value between them as an estimation of  the local Lipschitz constant.
     
     \paragraph{(D):}
     We used DeepFool~\citep{DeepFool}.

    \subsection{Experiment~\ref{subsec:large_network}}
    \label{subsec:appendix-large}

    Wide residual network \citep{WideResNet} with 16 layers and a width factor $k=4$ was used.
    We sampled 10000 images from an extra data available for SVHN dataset as validation data and combined the rest with the official training data, following \citet{Parseval}.
    All inputs were preprocessed so that each element has a value in a range $0$-$1$.

    Models were trained with Nesterov Momentum \citep{Nesterov} for $160$ epochs with a batch size of $128$.
    The initial learning rate was set to $0.01$ and it was multiplied by $0.1$ at epochs $80$ and $120$.
    For naive models, the weight decay with $\lambda=0.0005$ and the dropout with a dropout ratio of $0.4$ were used.
    For Parseval networks, the weight decay was removed except for the last fully-connected layer and Parseval regularization with $\beta=0.0001$ was added, following \citet{Parseval}.
    For a network with the spectral norm regularization, the weight decay was removed and the spectral norm regularization with $\lambda=0.01$ was used following \citet{SpectralNorm}.
    We note that both \citet{Parseval} and \citet{SpectralNorm} used batch-normalization for their experimental evaluations and thus, we left it for them.
    For LMT, we used $c=0.01$ and did not apply weight decay.
    In residual blocks, the Lipschits constant for the convolutional layer and the batch normalization layer was jointly calculated as described in \Sec~\ref{subsubsec:layers}.

  \section{Additional discussion}
  \label{sec:appendix-additional-discussion}

    \subsection{Application}
      \label{subsec:application}
      Since the proposed calculation method of guarded areas imposes almost no computational overhead at inference time, this property has various potential applications.
      First of all, we note that in real-world applications, even though true labels are not available, we can calculate the lower bounds on the size of perturbations needed to change the predictions.
      The primary use is balancing between the computational costs and the performance.
      When the provably guarded areas are sufficiently large, we can use weak and computationally cheap detectors of perturbations, because the detectors only need to find large perturbations.
      For data with small guarded areas, we may resort to computationally heavy options, e.g., strong detectors or denoising networks.

    \subsection{Improvements from Parseval networks}
      \label{subsec:difference}
      Here, we discuss the difference between our work and \citet{Parseval}.
      In the formulation of Parseval networks, the goal is to limit the change in some Lipschitz continuous loss by constraining the Lipschitz constant.
      However, since the existence of adversarial perturbations corresponds to the $0$-$1$ loss, which is not continuous, their discussion is not applicable.
      For example, if we add a scaling layer to the output of a network without changing its parameters, we can control the Lipschitz constant of the network.
      However, this does not change its prediction and this is irrelevant to the existence of adversarial perturbations.
      Therefore, considering solely the Lipschitz constant can be insufficient.
      In LMT, the insufficiency is avoided using Proposition~\ref{prop1} and ~\ref{prop2}.

      Additionally, we point out three differences.
      First, in Parseval networks, the upper bound of each component is restricted to be smaller than one.
      This makes their theoretical framework incompatible with some frequently used layers such as the batch normalization layer.
      Since they just ignore the effects of such layers, Parseval networks cannot control the Lipschitz constant of networks with normalization layers.
      On the other hand, our calculation method of guarded area and LMT can handle such layers without problems.
      Second, Parseval networks force all singular values of the weight matrices to be close to one, meaning that Parseval networks prohibit weight matrices to dump unnecessary features.
      As \citet{Theoretical} pointed out, learning unnecessary features can be a cause of adversarial perturbations, which indicates the orthonormality condition has adverse effects that encourage the existence of adversarial perturbations.
      Since LMT does not penalize small singular values, LMT does not suffer the problem.
      Third, LMT requires only differentiable bounds of the Lipschitz constants.
      This lets LMT be easily extended to networks with various components.
      On the other hand, the framework of Parseval networks requires special optimization techniques for each component.

    \subsection{Extensions of LMT}
      \label{subsec:extensions}
      The formulation of LMT is highly flexible, so we can consider some extended versions.
      First, we consider the applications that require guarded area different in classes.
      For example, to distinguish humans from cats will be more important than to classify Angora cats from Persian cats.
      In LMT, such knowledge can be combined by specifying different hyperparameter $c$ for each pair of classes.
      Second, we consider a combination of adversarial trainings.
      It will be more reasonable to require smaller margins for the inputs with large perturbations.
      In LMT, we can incorporate this intuition by changing $c$ according to the size of perturbations or merely set $c$ to zero for perturbed data.
      This ability of LMT to be easily combined with other notions is one of the advantages of LMT.

\end{document}